\documentclass[11pt,english]{article}

\usepackage[utf8]{inputenc}
\usepackage{amsmath}
\usepackage{amsthm}
\usepackage{soul}
\usepackage{amssymb}
\usepackage{amsfonts}
\usepackage[backend=biber, style=authoryear, maxbibnames=999, maxcitenames=1]{biblatex}
\usepackage{comment}
\usepackage{mathtools}
\usepackage{hyperref}
\usepackage{algorithm}
\usepackage{algpseudocode}
\usepackage{graphicx}
\usepackage[font=footnotesize]{caption}
\usepackage{xcolor}
\usepackage{enumitem}
\usepackage{nicefrac}
\usepackage{csquotes}

\newtheorem{assumption}{Assumption}
\newtheorem{remark}{Remark}
\newtheorem{theorem}{Theorem}

\newtheorem{lemma}{Lemma}
\newtheorem{definition}{Definition}
\newtheorem{example}{Example}
\newtheorem{proposition}{Proposition}

\usepackage[T1]{fontenc}
\usepackage{babel}

\author{
  Dragana Bajovic\\
  Faculty of Technical Sciences, University of Novi Sad, Novi Sad, Serbia\\
  \texttt{dbajovic@uns.ac.rs}
  \and 
  Dusan Jakovetic\\
  Faculty of Sciences, University of Novi Sad, Novi Sad, Serbia\\
  \texttt{dusan.jakovetic@dmi.uns.ac.rs}
  \and
  Soummya Kar\\
  Carnegie Mellon University, Pittsburgh, PA, USA\\
  \texttt{soummyak@andrew.cmu.edu}
\thanks{ 
This work is supported by the European Union’s Horizon 2020 Research and Innovation program under grant agreement No 957337. The paper reflects only the view of the authors and the Commission is not responsible for any use that may be made of the information it contains.}
}

\title{Large deviations rates for stochastic gradient descent with strongly convex functions}
\date{\vspace{-7ex}}

\begin{document}
\maketitle

\begin{abstract}
Recent works have shown that high probability metrics with stochastic gradient descent (SGD) exhibit informativeness and in some cases advantage over the commonly adopted mean-square error-based ones. In this work we provide a formal framework for the study of general high probability bounds with SGD, based on the theory of large deviations. The framework allows for a generic (not-necessarily bounded) gradient noise satisfying mild technical assumptions, allowing for the dependence of the noise distribution on the current iterate. Under the preceding assumptions, we find an upper large deviations bound for SGD with strongly convex functions. The corresponding rate function captures analytical dependence on the noise distribution and other problem parameters. This is in contrast with conventional mean-square error analysis that captures only the noise dependence through the variance and does not capture the effect of higher order moments nor interplay between the noise geometry and the shape of the cost function.  We also derive exact large deviation rates for the case when the objective function is quadratic and show that the obtained function matches the one from the general upper bound hence showing the tightness of the general upper bound. Numerical examples illustrate and corroborate theoretical findings. 
\end{abstract}

\section{Introduction}

The large deviations theory represents a well-established principled approach for studying \emph{rare events} that occur with stochastic processes, e.g.,~\parencite{DemboZeitouni93}. Typically, 
 we are concerned with a sequence of rare events $E_k$ related with the stochastic process of interest, 
 indexed by, e.g., time~$k$. In this setting, the probability of  event $E_k$,  $k=1,2,...$ typically decays exponentially in~$k$; the large deviations theory then enables to quantify this exponential rate.  Such an approach has found many applications in statistics~\parencite{Bucklew90},  mechanics~\parencite{Touchette2009LDStatMechs}, communications~\parencite{SchwartzWeiss95}, and information theory~\parencite{Cover91}.

 To be more concrete, consider an example of a sequence of random vectors 
 $X_k$ taking values in ${\mathbb R}^d$ that converge, e.g., almost surely, to 
 a (deterministic) limit point $x^\star \in {\mathbb R}^d$. 
 The rare event of interest $E_k$ can then be, for example, 
 $E_k =\{ \|X_k-x^\star\|\geq \delta\}$, for some 
 positive quantity $\delta$, with $\|\cdot\|$ denoting the Euclidean norm. 
  Equivalently, $E_k$ can be represented as $\{X_k \in C_{\delta}\}$, where $C_{\delta}$ is the complement of the $l_2$ ball of radius  $\delta$ centered at $x^\star$. Large deviations analysis then aims at discovering the corresponding rate of decay, i.e., the inaccuracy rate $\mathbf I(C_\delta)$:
\begin{equation}
\label{eq-rate-objective}
\mathbb P\left(X_k\in C_{\delta}\right) = e^{-k\, \mathbf I(C_\delta)+o(k)},
\end{equation}
where $o(k)$ denotes terms growing slower than linearly with $k$. The inaccuracy rate $\mathbf I(C_\delta)$ can usually be expressed via the so called \emph{rate function} $I: \mathbb R^d \mapsto \mathbb R$, according to the following formula~\parencite{Bahadur60}:
\begin{equation}\label{eq-set-fcn-via-rate-fcn}
\mathbf I(C_\delta) = \inf_{x\in C_{\delta}} I(x).
\end{equation}
Differently from the set function $\mathbf I$, the rate function $I$ does not depend on the region $C_\delta$; that is, when $C_\delta$ changes, only the region over which we minimize in~\eqref{eq-set-fcn-via-rate-fcn} changes, while the function remains unchanged. Furthermore, this  is true for arbitrary set $C_{\delta}$. This means that, once the rate function is computed, the corresponding inaccuracy rate can be obtained  via~\eqref{eq-set-fcn-via-rate-fcn} for a new given region of interest.

In this paper, we are interested in applying the large deviations theory to analyzing the stochastic gradient descent (SGD) method. SGD is a simple but widely used optimization method that finds numerous practical applications, such as  training machine learning and deep learning models, e.g.,
\parencite{niu2011hogwild,gorbunov2020unified,lei2020adaptivity}. More precisely, we consider unconstrained optimization problems where the goal is to minimize a smooth, strongly convex function $f: \,{\mathbb R}^d \rightarrow \mathbb R$, via the SGD method of the form:
\begin{equation}
    \label{eqn-SGD-basic}
    X_{k+1} = X_k - \alpha_k \,(\nabla f (X_k)-Z_k).
\end{equation}
Here, $k=1,2,...$ is the iteration counter, $\alpha_k=a/k$, $a>0$ is the step-size, and $Z_k$ is a zero-mean 
gradient noise that may depend on~$X_k$. In this context, we are interested in 
 solving for \eqref{eq-rate-objective} and \eqref{eq-set-fcn-via-rate-fcn} for the 
 SGD method~\eqref{eqn-SGD-basic}, where now $x^\star$ is interpreted as the (deterministic) global minimizer of~$f$. In other words, 
 we are interested in finding (or approximating) the rate function~$I(x)$ 
 that quantifies the ``tails'' or ``rare events'' of how the SGD sequence iterates
  $X_k$ deviate from the solution~$x^\star$. 

Clearly, evaluating~\eqref{eq-set-fcn-via-rate-fcn} for SGD is of significant interest. It readily provides insights into the high-probability bounds for SGD that have been subject of much research effort recently, \parencite{REF17,REF18,REF23,gorbunov2020stochastic,davis2021low}. However, unlike the typical high probability bound studies, the large deviations approach here  is fully flexible with respect to the choice of set $C_{\delta}$; e.g., the $l_2$-ball complement may be replaced with an arbitrary open set, such as $l_p$ norm complement of an arbitrary $l_p$-norm. 
 While large deviations theory is a well-established field, there has been a limited body of work that applies large deviations to the analysis of SGD. 
  Reference~\parencite{Woodroofe} is concerned with large deviations analysis for a scalar stochastic process equivalent to SGD in one dimension. 
  The authors of \parencite{WuLDP} study large deviations of SGD when the step-size converges to zero; however, they are not concerned with large deviations when 
  the iteration counter~$k$ increases -- the case of our interest here.

\textbf{Contributions}. In this paper, we are interested in evaluating 
 the large deviations rates in~\eqref{eq-rate-objective}
  and~\eqref{eq-set-fcn-via-rate-fcn} 
 for the SGD method, when the objective function $f$ is smooth and strongly convex.  Our main contributions are as follows. When 
  $f$ is a (strongly convex) quadratic function, 
  we establish the so-called full large deviations principle for 
  the sequence~$X_k$. This means that we evaluate rate function $I(x)$
   exactly, i.e., the corresponding rare event probability is computed exactly, 
   with upper and lower bounds matched, up to exponentially decaying factors. 
   We further explicitly quantify the rate function $I(x)$ as a function 
   of the distribution of the gradient noise. This reveals 
   a significant influence of higher order moments on the 
   performance (in the sense of rare event probabilities)
    of SGD. This is in contrast with conventional SGD analyses, 
    that typically capture only 
    the dependence on the gradient noise variance. 
  The large deviations principle for quadratic functions is established 
  under a very general class of gradient noise distributions that are essentially only required to have a finite moment generating function. 
  Next, for generic smooth and strongly convex costs $f$, 
  we establish a large deviations upper bound (a lower bound on function $I(x)$)
 that certifies an exponential decay of the rare event probabilities in \eqref{eq-rate-objective} with SGD. This is achieved when the distribution of 
 the gradient noise is sub-Gaussian. We further show that the obtained large deviations upper bound is tight, as the corresponding rate function actually matches, up to higher order factors, the exact rate function that we formerly establish for the quadratic costs.
 
Our results are related with high probability bounds-type studies of SGD and related stochastic methods  \parencite{Harvey2019,REF17,REF18,REF23,gorbunov2020stochastic}. Therein, for a given $\delta>0$ and a confidence level $1-\beta$, 
$\beta \in (0,1)$, the goal is to find $K(\delta,\beta)$
 such that $f(X_k)-f(x^\star) \leq \delta$ with probability at least $1-\beta$, for all 
$k \geq K(\delta,
 \beta)$. The works~\parencite{REF17,REF18,REF23,gorbunov2020stochastic} provide estimates of $K(\delta,\beta)$ that depend \emph{logarithmically} on~$\beta$. In more detail,~\parencite{REF17,REF18}  
 establish high probability bounds  for the stochastic gradient methods therein assuming sub-Gaussian gradient noises. The work~\parencite{REF23} calculates the corresponding bounds for the basic SGD and the mirror descent that utilize a gradient truncation technique, while    relaxing the noise sub-Gaussianity.   
    The work~\parencite{gorbunov2020stochastic} 
    establishes high probability bounds 
    for an accelerated SGD that also utilizes a clipping nonlinearity.
      The large deviations rates 
      in \eqref{eq-rate-objective} and \eqref{eq-set-fcn-via-rate-fcn}
       - give estimates of $K(\delta,\beta)$ that also 
       depend logarithmically on $\beta$, when 
       $\beta$ is small (goes to zero).\footnote{It is easy to see 
       this by noting that, for $\mu$-strongly convex costs, we have 
       $f(x) - f(x^\star) \geq \frac{\mu}{2}\|x-x^\star\|^2$, for 
       all $x \in {\mathbb R}^d$, 
       requiring that the the right hand side of  \eqref{eq-rate-objective} 
       be less than $\beta$, and reverse-engineering the smallest iterate $k$ for which 
      the latter holds.} 
      
      Compared with existing high probability bound works,  
      our results give the \emph{exact} (tight) exponential decay rate in~(2), and for an \emph{arbitrary set} that does not contain $x^\star$, not only the Euclidean ball complements. To be concrete, the closest results to ours are obtained in~\parencite{Harvey2019}. While they are not directly concerned with obtaining large deviations rates, their results (with some additional work) lead to an exponential decay rates for Euclidean ball complements. In contrast, our results work for arbitrary open sets. Furthermore, focusing only on Euclidean ball complements, our results provide much tighter exponential rate bounds. Specifically, as we show in the paper, the exponential rate that we provide captures the interplay between the noise geometry and the cost function curvature, see Section~\ref{subsec-rate-comparison} for details. From the technical perspective, this is achieved by working directly with the SGD iterates, as opposed to working with the distance of the iterates from the solution. To do so, we derive a novel set of techniques that build upon the large deviations theory rather than on martingale concentration inequalities.   
        
 The current paper is also related with 
 large deviations analyses of stochastic processes that arise with 
 distributed inference, such as estimation and detection. 
 Distributed detection has been studied in~\parencite{GaussianDD}, for Gaussian observations, and in~\parencite{Non-Gaussian-DD}, for generic observations. The work~\parencite{MBMS16InfTheory} evaluates  large deviations of the local states with a distributed detection method, when the step size parameter decreases. Reference~\parencite{MBMS16Refined} further analyzes the non-exponential terms and consider directed networks for a similar problem. The paper~\parencite{MaranoSayed19OneBit}  considers distributed detection with 1-bit messages. \parencite{Ping22} consider social learning problems.  
 Reference~\parencite{Bajovic22} analyzes large deviations for distributed estimation and social learning. Unlike these works on distributed inference, we are not directly concerned with distributed systems; also, the cost functions that we consider are more general and, unlike the works above, do not result in linear (distributed averaging) dynamics; hence, novel tools for large deviations analysis are required here.



The rest of the paper is organized as follows. Section~2 
explains the problem that we consider and gives  the required preliminaries. 
Section 3 provides the main results of the paper -- a large deviations upper bound for generic costs, and the full (exact) large deviations rates for quadratic costs. Specializing to the Gaussian noise, Section 4 provides analytical, closed-form expressions for the large deviations rate function. Section 5 gives the proof of the main lemma underlying the upper bound for the general functions.    Finally, we conclude in Section 6. Appendix contains additional insights and examples, numerical results, and missing proofs.

\section{Setup and preliminaries}
\label{sec-Setup}
We consider unconstrained optimization problem of the form
\begin{equation}
\label{eq-main-opt}
\min_{x\in \mathbb R^d} f(x).    
\end{equation}
 We assume that $f$ is $L$-smooth and $\mu$-strongly convex, and that the  stepsize in algorithm~\eqref{eqn-SGD-basic} is of the form $\alpha_k=a/(k+b),$ where $a,\,b>0.$ 

\begin{assumption}
\label{assum-L-mu}
We assume that $f$ is twice differentiable, $L$-smooth and $\mu$-strongly convex, where $0<\mu \leq L$. 
\end{assumption}
Strong convexity implies uniqueness of the solution of~\eqref{eq-main-opt}, which is denoted by $x^\star.$ We make the following assumption regarding the stepsize parameter $a.$
\\
\begin{assumption}
\label{assum-a-mu}
The stepsize parameter $a$ satisfies $a\mu>1$. 
\end{assumption}

Assumption~\ref{assum-L-mu} is standard in the analysis 
of optimization methods, i.e., it corresponds to a standard class of functions over which an optimization method analysis is carried out.  Assumption \ref{assum-a-mu}
 is required for some asymptotic arguments ahead, as $k \rightarrow \infty$. In practice, it may be restrictive that the constant $a$ is too large  in 
 the step-size choice $a/k$, as at the initial iterations (small $k$'s), we would have very large step-sizes. 
 This is alleviated by having an appropriately chosen constant $b>1$. 
 
We denote by $\tilde g(X_k)$ the stochastic gradient of $f$ returned by the gradient oracle at the current iterate $X_k,$  and by $g(X_k)$ the (exact) gradient of $f$ at the current iterate $X_k.$ The difference between $\tilde g(X_k)$ and $g(X_k)$ (the gradient ``noise'') is denoted by $Z_k = g(X_k) - \tilde g(X_k)$. We make the following assumptions on $Z_k.$

\begin{assumption}
\label{assumption-noise}
\begin{enumerate}
    \item \label{ass-noise-only-present-matters} For each $k$,  $Z_k$ depends on the past iterates only through $X_k$. 
    \item \label{ass-noise-only-realization-matters} For each $k$, the distribution of $Z_k$ given $X_k$ depends on $X_k$ only through its realization and does not depend on the current iterate index, $k$. 
    \item \label{ass-noise-zero-mean} For any given $x,$ $\mathbb E [Z_k|X_k=x] = 0,$ i.e., conditioned on the current iterate, the noise is zero-mean.  
\end{enumerate}
\end{assumption}

Assumption \ref{assumption-noise} allows for 
a general gradient noise that may actually depend on the current iterate $X_k$. This is a more general setting 
than the frequently studied case when $Z_k$ is i.i.d. and independent of $X_k$. Item 3. of Assumption~\ref{assumption-noise} says that, conditioned on the current iterate, the noise is zero-mean on average. 
This is also a standard bias-free noise assumption. 
Finally, note that items 1. and 2. in Assumption~\ref{assumption-noise} typically hold in machine learning settings. Therein, the goal is typically to minimize a population loss $f(x) = \mathbb{E}[\phi(x,v)]$ where the expectation is taken over the distribution of the data $v$, and $\phi$ is an instantaneous loss function. 
 Given that, at some iteration $k$, $X_k$
  takes a value $x$, the gradient noise equals 
$\nabla \phi(x,v_k) - \mathbb{E}[\nabla \phi(x,v)]$, where $v_k$ is the data point sampled at iteration~$k$. 
Then, items 1. and 2. are clearly satisfied,
 provided that the data sampling process is independent of the evolution of~$X_k$.

For $x\in \mathbb R^d,$ we denote by $H(x)$ the Hessian matrix of $f$ computed at $x.$ For short, we denote $H^\star=H(x^\star)$, i.e., $H^\star$ is the Hessian matrix of $f$ computed at $x^\star.$ For any $x\in \mathbb R^n$, define $h: \mathbb R^n\mapsto \mathbb R^d$ as the residual of the first order Taylor's approximation of the gradient $g$ at $x^\star$,
\begin{equation}
\label{def-residual}
h(x) = g(x) - H^\star(x-x^\star),    
\end{equation}
for $x \in \mathbb R^n$. For each $\delta>0,$ define also 
\begin{equation}
\label{def-max-residual-delta}
\overline h(\delta) = \sup_{x \in \mathbb B_{x^\star}(\delta)}\|h(x)\|,    
\end{equation}
where $B_x(\delta)$ denotes the Euclidean ball in $\mathbb R^d$  of radius $\delta\geq 0,$ centered at $x.$ The following result holds by a well-known corollary of Taylor's remainder theorem.
\begin{lemma}
\label{lemma-Taylor-small-remainder}
There holds $\overline h(\delta) = o(\delta),$ i.e., $\lim_{\delta \rightarrow 0} \frac{\overline h(\delta)}{\delta}=0.$
\end{lemma}

\begin{remark}
\label{remark-h-for-quadratic-is-zero}
Clearly, when $f$ is quadratic, $H(x)$ is constant for all $x\in \mathbb R^d$ and equal to $H^\star,$ implying $h(x) \equiv 0$ and also $\overline h(\delta)\equiv 0.$
\end{remark}

\begin{remark}
\label{remark-Lipschitz-Hessian}
Quantity $h(x)$ can be explicitly characterized if, in addition, it is assumed that the Hessian of function $f$ is Lipschitz continuous, i.e., if $\left\|H(x) - H(y) \right\| \leq L_H\, \|x-y\|$, for all $x,y \in {\mathbb R}^d$, for some nonnegative constant $L_H$. It is easy to show that, in this case, we have $\|h(x)\| \leq L_H\,\|x-x^\star\|^2$, for any $x \in {\mathbb R}^d.$ The latter implies a quadratic upper bound in $\delta$ on $\overline h(\delta)$, i.e., $\overline h(\delta) \leq L_H \delta^2,$ for each $\delta\geq 0.$ 
\end{remark}

\subsection{Distance to solution recursion} For analytical purposes, it is of interest to study the squared distance to solution of the current iterates $\|X_k-x^\star\|^2.$ To characterize the evolution of this quantity, we use standard arguments that follow from strong convexity and Lipschitz smoothness: 
\begin{align}
\label{eq-recursion-xi}
\|X_{k+1}-x^\star\|^2 & \leq \left(1-2\alpha_k \mu + 2 \alpha_k^2 L^2\right) \|X_k-x^\star\|^2 + 2 \alpha_k (X_k-x^\star)^\top Z_k \nonumber\\
& \,\,\,+ 2 \alpha_k^2 \|Z_k\|^2;
\end{align}
details of the derivations can be found in Appendix~A.

We introduce function $\beta_{k}: \mathbb R^2\mapsto \mathbb R,$ defined by $\beta_k(u,v)= 1-\alpha_k u + \alpha_k^2 v.$ Similarly, for any two iteration indices $l\leq k,$ we define $\beta_{k,l}:\mathbb R^2\mapsto\mathbb R$ by $\beta_{k,l}(u,v) = \beta_k(u,v) \cdots\beta_l(u,v).$ The following technical lemma providing bounds on the product functions $\beta_{k,l}$ will be useful for the study of recursion~\eqref{eq-recursion-xi} as well as other similar recursions that will emerge from the analysis. 

\begin{lemma}  
\label{lemma-beta-bounds}
Let $l$ and $k$ be two iteration indices such that $l<k$. For any nonnegative $u,\,v\in \mathbb R,$ and $\alpha_k = a/(k+b),$ where $b\geq 1,$ there holds:
\begin{enumerate}
    \item 
    \label{part-beta-UB} $\beta_{k,l} (u,v) \leq \left(\frac{l+b}{k+b+1}\right)^{au} e^{\frac{a^2 v}{l+b-1}}$;
    \item
    \label{part-beta-LB} for each $l$ such that $l+b\geq \frac{5a u}{2}$, there holds $\beta_{k,l} (u,v) \geq \left(\frac{l+b-1}{k+b}\right)^{au} e^{- \frac{a^2 u^2}{l+b-1}}; $
    \end{enumerate}
\end{lemma}
The proof of Lemma~\ref{lemma-beta-bounds} is given in Appendix~A.

Finally, for each iteration index $k$, we denote by $\mu_k$ the Borel measure on $\mathbb R^d$ induced by $X_k$. Similarly, we denote by $\nu_k$ the Borel measure induced by $\|X_k-x^\star\|.$

\subsection{Large deviations preliminaries}
\label{subsec-LD-analysis}
We next give a definition of the rate function and the large deviations principle. 

\textbf{Rate function~$I$ and the large deviations principle}. 
\begin{definition}[Rate function $I$~\parencite{DemboZeitouni93}]
\label{def-Rate-function}
Function $I:\mathbb R^d\mapsto [0,+\infty]$ is called a \emph{rate function} if it is lower semicontinuous, or,
equivalently, if its level sets are closed. If, in addition, the level sets of $I$ are compact (i.e., closed and bounded), then $I$ is called a good rate function.
\end{definition}
\begin{definition}[The large deviations principle~\parencite{DemboZeitouni93}]
\label{def-LDP}
Suppose that $I:\mathbb R^d\mapsto [0,+\infty]$ is lower semicontinuous. A sequence of measures $\mu_k$ on $\left(\mathbb R^d,\mathcal B\left(\mathbb R^d\right)\right)$, $k\geq 1$, is said to satisfy the large deviations principle (LDP) with rate function~$I$ if, for any measurable set $D\subseteq \mathbb R^d$, the following two conditions hold:
\begin{enumerate}
\item\label{eqn-LDP-UB}
$\displaystyle \limsup_{k\rightarrow +\infty}\,\frac{1}{k}\,\log\mu_k(D)\leq -\,\inf_{x\in \overline D} I(x);$
\item\label{eqn-LDP-LB}
$\displaystyle \liminf_{k\rightarrow +\infty}\,\frac{1}{k}\,\log\mu_k(D)\geq -\,\inf_{x\in D^{\mathrm{o}}} I(x).$
\end{enumerate}
\end{definition}

\textbf{Log-moment generating functions of the noise $Z_k$ and the iterates $X_k.$}
Following Assumption~\ref{assumption-noise}, we define the conditional LMGF of $Z_k$ given the last iterate $X_k.$
\begin{definition} [Conditional LMGF of $Z_k$]
\label{def-cond-MGF-and-LMGF-Z-k}
We denote by $\Lambda(\cdot;x)$ the log-moment generating function  (LMGF) of $Z_k$ given $X_k=x$,
\begin{equation}
\label{eq-lmgf-Z-k}
\Lambda (\lambda;x):= \log \mathbb E \left[\left.e^{\lambda^\top Z_k} \right| X_k=x\right],\:\:\mathrm{for\:\:}  \lambda, x\in \mathbb R^d. 
\end{equation}    
\end{definition}
It will also be useful to define the conditional moment-generating function of $\|Z_k\|^2$, which we denote by $M(\cdot;x)$: 
\begin{equation}
\label{eq-mgf-Z-k-norm}
M(\nu;x):= \mathbb E \left[\left.e^{\nu \|Z_k\|^2} \right| X_k=x\right],  \end{equation}
for $\nu\in \mathbb R,$ $x\in \mathbb R^d$. By the inequality $e^x \leq x+e^{x^2},$ which holds for all $x
\in \mathbb R,$ we have $\mathbb E\left[\left.e^{\lambda^\top Z_k}\right|X_k\right]\leq \mathbb E[\lambda^\top Z_k|X_k] + \mathbb E\left[e^{(\lambda^\top Z_k)2}|X_k\right] \leq \mathbb E\left[\left.e^{ \|\lambda^2\| \|Z_k\|^2} \right| X_k\right],$ where we used the Cauchy-Schwartz inequality, for the second term, and the fact that $Z_k$ is zero-mean, for the first term. Thus,
\begin{equation}
\label{relation-Lambda-M}
\Lambda(\lambda;x) \leq \log M(\|\lambda^2\|;x)
\end{equation}
for any realization $x$ of $X_k.$

Lemma~\ref{lemma-properties-Lambda} lists properties of $\Lambda$ that will be used in the paper.
\begin{lemma}[Properties of $\Lambda$]
\label{lemma-properties-Lambda}
For any given $x\in \mathbb R^d$ the following properties hold:
\leavevmode
\begin{enumerate}
    \item \label{part-Lambda-cvxity} $\Lambda (\cdot; x)$ is convex and differentiable in the interior of its domain;
    \item \label{part-Lambda-at-0} $\Lambda(0; x)=0$ and $\nabla \Lambda (0; x)=\mathbb E[Z_k|X_k=x]=0$;
    \item \label{part-Lambda-nonnegative}
    $\Lambda(\lambda; x) \geq 0,$ for each $\lambda.$
\end{enumerate}
\end{lemma}

\begin{proof}
Convexity and differentiability are general properties of log-moment generating functions~\parencite{DemboZeitouni93}, as well as the zero value at the origin property and also that the gradient at the origin equals the mean vector; $\nabla \Lambda(0;x)=0$ follows by the assumption that the noise is zero-mean, Assumption~\ref{ass-noise-zero-mean}. The non-negativity from Part~\ref{part-Lambda-nonnegative} follows by invoking convexity and exploiting the two properties from part~\ref{part-Lambda-at-0}, i.e., for any $x\in \mathbb R^d$:
$\Lambda(\lambda; x) \geq \Lambda(0;x) + \nabla \Lambda(0;x)^\top \lambda = 0.$
\end{proof}

\begin{example}
\label{example-LMGF-Gaussian}
To illustrate the LMGF function $\Lambda,$ we consider the case when, conditioned on an arbitrary realization $X_k=x,$ the gradient noise $Z_k$ is Gaussian, with mean vector equal to zero vector and covariance matrix $\Sigma (x).$ Using standard formula for the LMGF of a Gaussian multivariate, we have
\begin{equation}
\label{eq-LMGF-Z-k-Gaussian}
\Lambda(\lambda; x) = \frac{1}{2} \lambda^\top S (x) \lambda,  \end{equation}
for $\lambda \in \mathbb R^d.$ We note that when the gradient noise $Z_k$ is independent of the current iterate $X_k$, the indices $X_k$ in the preceding formula can be omitted, i.e., the expression for $\Lambda$ simplifies to $\Lambda(\lambda; X_k)= \frac{1}{2} \lambda^\top S \lambda,$ for all realizations $X_k.$
\end{example}

It will also be of interest to define the (unconditional) log-moment generating function of the iterates $X_k$.
\begin{definition} [LMGF of $X_k-x^\star$]
\label{def-LMGF-X-k}
We let $\Gamma_k$ denote the (unconditional) moment generating function of $X_k$,
\begin{equation}
\label{eq-mgf-X-k}
\Gamma_k(\lambda):= \mathbb E \left[e^{\lambda^\top (X_k-x^\star)}\right],      
\end{equation}
for $\lambda \in \mathbb R^d$. The (unconditional) log-moment generating function of $X_k$ is then given by $\log \Gamma_k$.  
\end{definition}

We assume that the initial iterate $X_1$ is deterministic\footnote{We note that this assumption can be relaxed to allow for random initial iterate; see Appendix~D for details.}. Hence, $\Gamma_1$ is finite for all $\lambda \in \mathbb R^d$.

We assume that the family of functions $\Lambda(\cdot;x)$ satisfy the following regularity conditions.  
\begin{assumption} [Lipschitz continuity in $x$]
\label{assumption-Lambda-smooth}
There exists a constant $L_{\Lambda}$ such that for every $\lambda$, $x$, $y\in \mathbb R^d$, there holds:
\begin{equation}
\label{eq-condition-Lambda-smooth}
|\Lambda(\lambda;x)-\Lambda(\lambda; y)|\leq L_{\Lambda} \|\lambda\|^2 \|x-y\|.    
\end{equation}
\end{assumption}
\begin{remark} We note that Assumption~\ref{assumption-Lambda-smooth} is trivially satisfied when the noise distribution does not depend on the current iterate. For another illustration, consider Gaussian random noise distribution from Example~\ref{example-LMGF-Gaussian}, for which we have:
\begin{align}
\Lambda(\lambda;x)-\Lambda(\lambda;y) & = \frac{1}{2}\lambda^\top (S(x)-S(y))\lambda \\
& \leq \frac{1}{2}\|\lambda\|^2\|S(x)-S(y)\|.
\end{align}
Comparing with the condition in~\eqref{eq-condition-Lambda-smooth},  we see that~\eqref{eq-condition-Lambda-smooth} is satisfied when entries of the covariance matrix $S$, as functions of $x$, are Lipschitz continuous. \end{remark}  

The assumption below will be used for the proof of the main result of the paper, when the case of general convex functions is considered.
\begin{assumption} [Sub-Gaussian noise]
\label{assumption-subGaussian}
There exists a constant $C_1>0$ such that, for each $\lambda,\,x\in \mathbb R^d$
\begin{equation}
\label{eq-subGaussian}
\Lambda(\lambda;x) \leq C_1 \frac{\|\lambda\|^2}{2}.    
\end{equation}
\end{assumption}
\begin{remark} Assumption~\ref{assumption-subGaussian} means that the gradient noise has ``light tails,'' i.e., 
there exist positive constants $c_1,c_2$, such that 
 the probability that the magnitude of the norm of the noise vector is above $\epsilon$ is upper bounded 
 by $c_1 \,e^{-c_2\epsilon^2}$, for any $\epsilon>0$.
  Clearly, a Gaussian zero-mean multivariate distribution satisfies this property, 
  and also any noise distribution with compact support.

This assumption also ensures that, for each given $\lambda,$ the value of the variance ``proxy'' $C_1$ cannot grow without bound as the domain of iterates $x$ enlarges. For a Gaussian distribution, this means that the variance, as a function of the current iterate should be uniformly bounded over the domain of the iterates, which is a typical assumption in related works.  
\end{remark}
 
We also use the following implications of Assumption~\ref{assumption-subGaussian}. 
\begin{proposition}
\label{proposition-subGaussian-implications}
\begin{enumerate}
    \item \label{part-Orlitz-2-property} There exists $C_2>0$ such that
    \begin{equation}
    \mathbb E\left[\left.\exp\left( \frac{\|Z_k^2\|}{C_2}\right)\right| X_k\right] \leq e.    
    \end{equation}
    \item \label{part-LMGF-Z_k-squared} For any $\nu \in [0,1/C_2]$ there holds
    \begin{equation}
M(\nu; X_k) \leq \exp(\nu C_2).    
    \end{equation}
\end{enumerate}  
\end{proposition}
\begin{proof}
The proof of part~\ref{part-Orlitz-2-property} can be derived by applying properties of sub-Gaussian random variables to $\|Z_k\|;$ see, e.g., Proposition 2.5.2 in~\cite{Vershynin2018} and also~\cite{Jin2019} for a treatment of sub-Gaussian random vectors.

To show part~\ref{part-LMGF-Z_k-squared}, fix $\nu \in [0,1/C_2].$ By H\"{o}lder's inequality (applied for $\text{``}p\text{''}=1/(\nu C_2) \geq 1$) 
\begin{align}
M(\nu ;X_k) & \leq 
\left(\mathbb E\left[\left.\exp\left( 1/C_2 \|Z_k\|^2\right)\right|X_k \right]\right)^{\nu C_2} \\
&  \leq \exp(\nu C_2)
\end{align}
where in the second inequality we used part~\ref{part-Orlitz-2-property}. 
\end{proof}

\begin{remark} 
\label{remark-subGaussian-Gaussian-special-case}
When the distribution of $Z_k$ is Gaussian, zero mean and with covariance matrix $\Sigma,$ and independent of the current iterate, we have
\begin{equation}
\Lambda(\lambda) = \frac{1}{2}\lambda^\top \Sigma \lambda \leq \frac{1}{2} \sigma_{\max}^2 \|\lambda\|^2,     
\end{equation}
where $\sigma_{\max}^2$ is the maximal eigenvalue of $\Sigma.$ Comparing with Assumption~\ref{assumption-subGaussian}, we see that condition~\eqref{eq-subGaussian} holds with $C_1 = \sigma_{\max}^2.$ It can also be shown that part 1. of Proposition~\ref{proposition-subGaussian-implications} holds for $C_2\geq 2 \sigma_{\max}^2.$  
\end{remark}

\subsection{Key technical lemma}
\begin{definition}
\label{def-FL-transform}
The Fenchel-Legendre transform, or the conjugate, of a given function $\Psi:\mathbb R^d \mapsto \mathbb R$ is defined by
\begin{equation}
\label{def-Conjugate}
I(x)=\sup_{\lambda\in \mathbb R^d} x^\top \lambda - \Psi(\lambda),\:\:\mathrm{for\:\:}x\in \mathbb R^d.
\end{equation}
\end{definition} 

\begin{lemma}
\label{lemma-GE-auxilliary}
Let $\Psi_k$ be a sequence of log-moment generating functions associated to a given sequence of measures $\mu_k: \mathcal B (\mathbb R^d) \mapsto [0,1].$ Suppose that, for each $\lambda \in \mathbb R^d,$ the following limit exists: 
\begin{equation}
\label{eq-GE-condition}
\limsup_{k\rightarrow +\infty} \frac{1}{k} \Psi_k(k\lambda) \leq \Psi(\lambda).  \end{equation}
If $\Psi(\lambda)<\infty$ for each $\lambda \in \mathbb R^d,$ then the sequence $\mu_k$ satisfies the LDP upper bound with the rate function $I$ equal to the Fenchel-Legendre transform of $\Psi.$ If, in addition,~\eqref{eq-GE-condition} holds as a limit and with equality, then the sequence of measures satisfies the LDP with rate function $I.$ 
\end{lemma}
The second part of the lemma follows by the G\"{a}rtner-Ellis theorem. The first part can be proven by similar arguments as in the proof of the upper bound of the G\"{a}rtner-Ellis theorem; for details, see also the proof of Lemma~35 in~\parencite{Bajovic22}. 

\section{Large deviations rates for SGD iterates $X_k$}
\label{sec-LD-general}

\subsection{Large deviations rates for $\|X_k-x^\star\|$}
\label{subsec-LD-Xk-distance}
To derive the main result -- the large deviations rate function for the SGD sequence $X_k,$ we first study large deviations properties of the sequence $\|X_k-x^\star\|,$ $k=1,2,...$ For the latter, we first exploit the idea from~\cite{Harvey2019} to obtain a high probability bound for the (scaled) quantity $\|X_k-x^\star\|^2,$ via its moment generating function. We then use this bound to derive a rate function (bound) for $\|X_k-x^\star\|.$ Since our assumptions are distinct than those in~\cite{Harvey2019} (e.g., the recursive form that we work with here contains factors that require special treatment than the one in~\cite{Harvey2019}, also we do not assume bounded noisy gradient, as is the case with the proof available in~\cite{Harvey2019}), we provide full proof details, see appendix. 

\begin{lemma}
\label{lemma-HPB-Yk}
For any $k,$ there holds
\begin{equation}
\label{eq-HPB-Yk}
\mathbb P\left( \|X_k-x^\star\| \geq \delta \right) \leq e e^{-(k+k_0) B \delta^2},     
\end{equation}
where $B=\min \{\frac{1}{k_0 \|X_1-x^\star\|},\frac{2a\mu -1}{ 4 \max\{C_1, 2C_2\} a^2}\}$ and $k_0 = 4a^2L^2/(2 a\mu -1).$ 
\end{lemma}

\begin{remark}
\label{remark-on-thrm-LD-rate-for-xi-k}
The preceding theorem establishes a large deviations upper bound for the sequence of squared distance to solution iterates $\xi_k,$ by exploiting noise sub-Gaussianity. By its nature, this result is a rough characterization of the large deviations rate function for the sequence $X_k.$ In addition to being a result of independent interest, the utility consists in bounding the tails of distribution $\mu_k,$ as an enabling step towards deriving a fine, close to exact rate function for the SGD iterates $X_k,$ as the main contribution of this paper. The latter is the subject of the next section. 
\end{remark}

\subsection{Main result: Large deviations rates for $X_k$}
\label{subsec-LD-Xk}

We now present our result for general convex functions satisfying assumptions from Section~\ref{sec-Setup}. The pillar of the analysis is the limit of the sequence of log-moment generating functions $\log \Gamma_k$ of the SGD iterates.
\begin{lemma}
\label{lemma-LMGF-limsup}
Suppose that Assumptions~1-5 hold and that the stepsize is given by $\alpha_k=a/(k+k_0).$ For any $\lambda \in \mathbb R^d,$ 
\begin{equation}
\label{eq-LMGF-limsup}
\limsup_{k\rightarrow+\infty} \frac{1}{k} \log \Gamma_{k} (k \lambda) \leq \overline \Psi(\lambda):= \Psi^\star(\lambda)+r(\lambda),    
\end{equation}
where $\Psi^\star$ is defined by
\begin{equation}
\label{eq-Psi-star}
\Psi^\star (\lambda)= \int_0^1 \Lambda (a Q D(\theta)Q^\top \lambda; x^\star) d\theta,
\end{equation} 
where $H^\star=Q D Q^\top$, $Q Q^\top =I$, $D=\mathrm{diag} \{\rho_1,...,d_n\}$, $D(\theta) = \mathrm{diag} \{\theta^{a\rho_1-1},...,\theta^{ad_n-1}\}$, 
$r(\lambda) = \frac{4a^2  \overline \gamma^2 L_\Lambda   }{B^2}\|\lambda\|^4  + a \|\lambda\|\overline h\left(\frac{2 \overline \gamma \|\lambda\|}{B}\right),$  and $\overline \gamma = \max\{1, \sqrt{(1-a\mu)^2 + a^2 (L^2-\mu^2)}\}.$ 
\end{lemma}
The proof of Lemma~\ref{lemma-LMGF-limsup} is given in section~\ref{subsec-main-proof}. Having the limit in~\eqref{eq-LMGF-limsup}, LDP upper bound follows by Lemma~\ref{lemma-GE-auxilliary}. 
\begin{theorem}
\label{theorem-LDP-general}
Suppose that Assumptions~1-5 hold and that the stepsize is given by $\alpha_k=a/(k+k_0).$ Then, the sequence of iterates $X_k$ satisfies the LDP upper bound with rate function $\overline I$ given as the Fenchel-Legendre transform of $\overline \Psi$ from Lemma~\ref{lemma-LMGF-limsup}, i.e., for any closed set $F$:
\begin{equation}
\label{eq-LD-UB-general}
\limsup_{k\rightarrow +\infty} \frac{1}{k} \log \mathbb P\left(X_k \in F\right) \leq - \inf_{x + x^\star \in F} \overline I(x).    
\end{equation}
\end{theorem}

\begin{remark}
The rate function $\overline I$ depends on the Hessian matrix at the solution, $H(x^\star)$. However, coarser exponential rate bounds can be obtained by uniformly bounding the eigenvalues of $H(x^\star),$ as by our assumptions they are all confined in the interval $[\mu, L].$ See Appendix~D for details.   
\end{remark}

\subsection{Discussions and interpretations}
\subsubsection{Positivity of $\overline I$ and exponential decay} 
From the fact that $\Psi^\star,\, r \geq 0,$ and that both $\Psi^\star$ and $r$ are finite on $\mathbb R^d,$ it can be shown that $\overline I \geq 0$ and that $\overline I$ is a good rate function. Specifically, $\overline I(0)=0$ and $\overline I(x)>0$ for any $x\neq 0.$  Therefore, for any closed set $F$ such that $x^\star \notin F,$ we have
\begin{equation}
\inf_{x+x^\star \in F} \overline I(x)>0,    
\end{equation}
that is, the exponent in~\eqref{eq-LD-UB-general} is strictly positive ensuring the exponential decay of the probabilities $\mathbb P\left(X_k \in F\right)$.  To illustrate this in intuitive terms, we take as a special case the set $F = B_{x^\star}^{\mathrm{c}} (\delta),$ for some $\delta>0.$ Then, the event of interest becomes $\{X_k\in F\} =\{ \|X_k-x^\star\| \geq \delta\}.$ Thus, for any $\delta>0,$ Theorem~\ref{theorem-LDP-general} implies that
\begin{equation}
\label{eq-probability-epsilon}
\limsup_{k\rightarrow +\infty} \frac{1}{k} \log \mathbb P\left(\|X_k-x^\star\|\geq \delta\right) \leq - R(\delta),    
\end{equation}
where $R(\delta) = \inf_{\|x\|\geq \delta} I(x)>0.$ 
 
\subsubsection{Remainder term $r$}
\label{subsubsec-remainder}
Recalling Lemma~\ref{lemma-Taylor-small-remainder}, it is easy to see that $r(\lambda)= o(\|\lambda\|^2),$ i.e., $\lim_{\|\lambda\| \rightarrow 0} \frac{r(\lambda)}{\|\lambda\|^2} = 0.$ Also, for a function $f$ that has Lipschitz Hessian, see Remark~\ref{remark-Lipschitz-Hessian}, the residual function $r$ behaves roughly as $\sim \|\lambda\|^3$.

Further, for the special case when $f$ is quadratic, $\overline h(\delta) = 0,$ and hence $r$ contains only the first term, and thus $r(\lambda) \sim \|\lambda\|^4$. Similarly, when the noise distribution does not depend on the current iterate, we have that $L_{\Lambda}=0,$ and hence $r(\lambda)= o(\|\lambda\|^2).$ Finally, for the case when both of the preceding conditions hold, the residual term is zero at all points: $r\equiv 0,$ and hence the rate function $\overline I=I^\star,$  where $I^\star$ is the Fenchel-Legendre transform of $\Psi^\star.$

\subsubsection{Small deviations regime}
When high precision estimates are sought, or equivalently, for small $\delta$ in~\eqref{eq-probability-epsilon}, the candidate values of $\overline I$ in the minimization are very close to $0.$ By the fact that the remainder term $r(\lambda)=o(\|\lambda^2\|),$ it can be shown that, in the small deviations regime, $\overline I$ is determined by $\Psi^\star$ only, i.e., $\overline I \approx I^\star,$ and, also, its behaviour is dominantly characterized by the noise variance.     

\subsection{LDP for quadratic functions}
\label{subsec-LDP-quadratic}
In this section we provide the full LDP for the case when $f$ is a quadratic function. The proof of Theorem~\ref{theorem-quadratic-full-LDP} is given in Appendix~E.
 \begin{theorem}
\label{theorem-quadratic-full-LDP}
Suppose that the objective function $f$ is quadratic, that Assumptions~\ref{assum-a-mu}-\ref{assumption-noise} hold, with the step size given by $\alpha_k= a/k$. Suppose also that the noise distribution does not depend on the current iterate and that it has a finite log-moment generating function $\Lambda$. Then, the sequence $X_k$ satisfies the large deviations principle with the rate function $I^\star$ given as the conjugate of $\Psi^\star$ defined in~\eqref{eq-Psi-star}, with $\Lambda(\cdot;x^\star)$ replaced by $\Lambda.$  
\end{theorem}

The rate function $I^\star$ depends on the distribution of $Z_k$ and fully captures all moments of this distribution. In particular, for non-Gaussian distributions, it captures exactly the dependence not only on the variance, but also on higher order moments. 

\begin{remark}
We note that, in contrast with Theorem~\ref{theorem-LDP-general}, for Theorem~\ref{theorem-quadratic-full-LDP} the conditional distribution of $Z_k$ can be arbitrary, as long as $\Lambda$ is finite. In particular, it allows for distributions that are not light-tailed, such as Laplacian.    
\end{remark}

\begin{remark}
Recalling the discussion from subsection~\ref{subsubsec-remainder}, we see that the upper bound rate function from Theorem~\ref{theorem-LDP-general} and the rate function from Theorem~\ref{theorem-quadratic-full-LDP} match, hence showing that the bound in Theorem~\ref{theorem-LDP-general} is tight.
\end{remark}

\section{Gaussian noise: analytical characterization of the rate function}
\label{sec-analytical-Gaussian}

 If the noise $Z_k$ has a Gaussian distribution with mean value zero and covariance matrix $\Sigma$, then $\Psi^\star$ is computed by
\begin{equation}
 \Psi^\star (\lambda)= \frac{a^2}{2} \int_0^1 \lambda^\top Q D(\theta) Q^\top \Sigma  Q D(\theta) Q^\top \lambda d\theta.  \end{equation}

To simplify the notation, let $S= Q^\top \Sigma Q,$ and $M(\theta) = D(\theta) S D(\theta).$ It is easy to verify that $M_{ij}(\theta) = S_{ij} \theta^{a(\rho_i+\rho_j)-2},$ for any $i,j=1,...,d,$ and thus $\int_{0}^1 M_{ij}(\theta) d\theta = S_{ij}/(a(\rho_i+\rho_j)-1).$ Hence, we obtain the following closed-form expression for $\Psi^\star:$
\begin{equation}
\Psi^\star (\lambda)= \frac{a^2}{2} \lambda^\top Q S^\star Q \lambda,    
\end{equation}
where $ S_{ij}^\star = S_{ij}/(a(\rho_i+\rho_j)-1),$ for $i,j=1,...,d.$

Recalling the Definition~\ref{def-FL-transform}, it can be shown that the Fenchel-Legendre transform $I^\star$ of $\Psi^\star$ is given by 
\begin{equation}
I^\star(z) = \frac{1}{2 a^2} z^\top Q^\top  {S^\star}^{-1} Q z.    
\end{equation}

To obtain further intuition about the rate function $I^\star$, we consider the special case when the Hessian matrix $H^\star$ and the covariance matrix $\Sigma$ share the same eigenspace (given by the columns of the matrix $Q$). Intuitively, the latter means that the orientation of the quadratic approximation of $f$ at the origin is aligned with the gradient noise distribution in each of the axes.  In this case, it follows that $S = Q^\top \Sigma Q$ is diagonal with $S_{ii} = \sigma_{ii}^2,$ where $\sigma_{ii}^2$ is the $i$-th eigenvalue of $\Sigma$ (i.e., the eigenvalue of $\Sigma$ corresponding to its eigenvector given by the $i$-th column of matrix $Q$). It follows that $S^\star$ is also diagonal with $S_{ii}^\star = \sigma_{ii}^2/(2 a\rho_i-1).$ Thus, the following neat expression for the rate function $I^\star$ emerges:
\begin{equation}
\label{eq-rate-fcn-aligned-H-and-noise}
I(z) = \frac{1}{2 a^2} z^\top Q^\top  \mathrm{diag} \left( \frac{2 a\rho_1-1}{ \sigma_{11}^2},..., \frac{2 a\rho_d-1}{ \sigma_{dd}^2} \right)  Q z.    
\end{equation}
 
\subsection{Decay rates with $l_2$ balls}
\label{subsec-rates-for-l2-balls}

We consider the case when in the large deviations event of interest $\{X_k \in F\}$ the set $F$ is given as the complement of an $l_2$ ball around the solution $x^\star:$ $F = B^{\mathrm{c}}_{x^\star}(\delta),$ i.e., $\{X_k \in F\} = \{\|X_k-x^\star\| \geq \delta\|.$ Assuming that the residual is zero (see the result for quadratic functions in Section~\ref{subsec-LDP-quadratic}), by Theorem~\ref{theorem-LDP-general}, we have
\begin{equation}
\limsup_{k\rightarrow +\infty} \frac{1}{k}\log \mathbb P\left( \left\|X_k-x^\star \right\|\geq \delta \right) \leq \inf_{\|z\| \geq \delta} I(z) =: \mathbf I (B^{\mathrm{c}}_{x^\star}(\delta)).    
\end{equation}

For the Gaussian noise assumed in this section, we have:
\begin{align}
\label{eq-rate-ball-sets}
\mathbf I (B^{\mathrm{c}}_{x^\star}(\delta)) & = \inf_{\|z\|\geq \delta } \frac{1}{2 a^2} z^\top Q^\top  {S^\star}^{-1} Q z \nonumber \nonumber \\
& = \frac{\delta^2}{2 a^2} \inf_{\|w\|\geq 1}  w^\top Q^\top  {S^\star}^{-1} Q w \nonumber \\
& = \frac{\delta^2}{2 a^2} \frac{1}{ \lambda_{\max} (S^\star)},
\end{align}
where $\lambda_{\max} (S^\star)$ is the largest eigenvalue of the matrix $S^\star.$ Hence, to find the value of the exponent $\mathbf I$ for any given ball-shaped set, it suffices to find (once) the maximal eigenvalue of $S^\star,$ and the exponent $\mathbf I$ would be easily computed by the quadratic function~\eqref{eq-rate-ball-sets}.   

We close the analysis with a particularly elegant solution for the special case when $H^\star$ and $\Sigma$ are axes-aligned. As detailed at the beginning of the section,  in the latter case, $S^\star$ is diagonal, with $S_{ii}^\star=\sigma_{ii}^2/(2 a\rho_i-1),$ and the rate function is given by~\eqref{eq-rate-fcn-aligned-H-and-noise}. Thus, to find the maximal eigenvalue of $S^\star$ reduces to finding the index $i$ for which $ \frac{\sigma^2_{ii}}{2 a\rho_i-1}$ is highest, or, equivalently, $\frac{2a \rho_i-1}{\sigma^2_{ii}}$ the lowest, which then yields: 
\begin{align}
\label{eq-rate-ball-sets-aligned-H-and-noise}
\mathbf I (B^{\mathrm{c}}_{x^\star}(\delta)) = \frac{\delta^2}{2 a^2} \min \{ \frac{2a \rho_i-1}{\sigma^2_{ii}}: i=1,...d\},
\end{align}
where, we recall, $\rho_i$ is the $i$-th eigenvalue of $H^\star.$ What the expression above is saying is that, in order to find the exponential decay rate for an $l_2$ ball, we should search for the direction $i$ in which the value $\frac{\sigma^2_{ii}}{2 \rho_i-1}$ is highest.  In a sense, the latter quantity can be thought of as the effective noise variance, capturing the interplay between the noise distribution and the shape of the function at the solution. Specifically, if along the direction where the noise variance is highest, say $i^\star,$ the function has a high curvature (i.e., large $\rho_{i^\star}$), this will effectively alleviate the effects of noise and increase the rate function, in comparison to the case when the curvature along $i$ is lower, and therefore result in faster convergence.  

Finally, when the noise is isotropic, i.e., such that $\sigma_{ii}^2 = \sigma^2,$ for all $i,$ exploiting the fact that the spectrum of $H^\star$ lies inside the interval $[\mu, L],$ the rate function is found by:
\begin{align}
\label{eq-rate-ball-sets-isotropic-noise}
\mathbf I (B^{\mathrm{c}}_{x^\star}(\delta)) = \frac{\delta^2}{2 a^2}  \frac{2 a\mu -1}{\sigma^2 }.
\end{align}

\subsection{Comparison with the rate from Lemma~\ref{lemma-HPB-Yk}}
\label{subsec-rate-comparison}

We now compare the rate function bounds obtained from Lemma~\ref{lemma-HPB-Yk} and Theorem~\ref{theorem-LDP-general}. To gain deeper insights, we will assume that the residual term $r$ equals zero (compare with Section~\ref{subsec-LDP-quadratic}). We also assume that the noise is Gaussian and axes-aligned with the matrix $H^\star$ (see the preceding subsection). The exponent $B$ from~\ref{lemma-HPB-Yk} can be upper bounded by\footnote{The dependence in $B$ on $X_1$ in Lemma~\ref{lemma-HPB-Yk} seems to be an artifact of the conducted proof method, rather than an essential property of the exponential rate that Lemma~\ref{lemma-HPB-Yk} pursues. Hence, for unbiased comparison, we omit this factor in the analysis of the rate $B$.} 
\begin{equation*}
B \leq \frac{2a\mu -1}{ 4 {\sigma_{\max}}^2  a^2},    
\end{equation*}
where we exploited the fact that, for Gaussian noise, $C_1 = \sigma_{\max}^2,$ see Remark~\ref{remark-subGaussian-Gaussian-special-case}.
Hence, for an $l_2$ ball of radius $\delta,$ the exponent that Lemma~\ref{lemma-HPB-Yk} provides is bounded by
\begin{equation}
\label{eq-rate-Lemma-HPB}
B \delta^2\leq \frac{\delta^2}{4 a^2} \frac{2a\mu -1} {\sigma_{\max}^2 }.     
\end{equation}
The counterpart obtained from Theorem~\ref{theorem-LDP-general} is given by expression~\eqref{eq-rate-ball-sets-aligned-H-and-noise}. To show direct comparison with~\eqref{eq-rate-Lemma-HPB}, we further upper bound this value by decoupling the minimization over $i:$  
\begin{align}
\label{eq-rate-Theorem-general}
\mathbf I (B^{\mathrm{c}}_{x^\star}(\delta))&  =  \frac{\delta^2}{2 a^2} \min \{ \frac{2a \rho_i-1}{\sigma^2_{ii}}: i=1,...d\} \nonumber\\
& \geq \frac{\delta^2}{2 a^2} \frac{\min \{2a \rho_i-1: i=1,..,d\} }{\max \{\sigma_{ii}^2: i=1,..,d\}} \nonumber \\
& = \frac{\delta^2}{2 a^2} \frac{2a\mu -1}{\sigma_{\max}^2}.
\end{align}
Comparing with~\eqref{eq-rate-Lemma-HPB}  (and ignoring the scaling constant $2$), the following important point can be noted: on intuitive level, the derivation of the rate $B$ is equivalent to that of decoupling the effects of the noise distribution and the shape of the function $f$ at the origin with the rate $I^\star$. Hence, in contrast with $I^\star,$ the rate $B$ is oblivious to the interplay between these two quantities -- from a purely technical perspective, this distinction is a consequence of relying on recursions on the iterates' distance to solution, $\|X_k-x^\star\|,$ as opposed to working directly with the iterates $X_k,$ as is the case in the proof of Theorem~\ref{theorem-LDP-general}. 

\section{Proof of Lemma~\ref{lemma-LMGF-limsup}}
\label{subsec-main-proof}
This section provides the main elements of the proof of the limit in~\eqref{eq-LMGF-limsup}; the proofs of omitted results can be found in Appendix~C. Fix $\lambda \in \mathbb R^d.$ Fix $k\geq 1.$ Define $\eta_l= B_{k,l} \eta_k,$ $B_{k,l}=(I-\alpha_l H^\star)\cdots (I-\alpha_k H^\star),$ $\eta_k= k \lambda.$ By Lemma~\ref{lemma-beta-bounds},  
\begin{equation}
\label{eq-eta-l-bound}
\|\eta_l \| \leq k \left(\frac{l+k_0}{k+k_0+1}\right)^{a \mu} \|\lambda\| \leq (l+k_0) \|\lambda\|.  
\end{equation}
For an arbitrary $l\leq k,$ there holds
\begin{align}
\Gamma_{l+1}(\eta_l) & = \mathbb E\left[ \exp\left(\eta_l^\top (X_{l+1}-x^\star)\right)\right]\nonumber \\
& = \mathbb E\left[\mathbb E\left[ \exp\left(\eta_l^\top (X_l - \alpha_l g(X_l)+\alpha_l Z_l - x^\star) \right) | X_l\right] \right] \nonumber \\
& = \mathbb E\left[ \exp\left(\Lambda (\alpha_l \eta_l; X_l)+ \eta_l^\top (X_l-\alpha_l g(X_l)-x^\star)\right) \right]\nonumber \\
& = \int_{x\in \mathbb R^d} \Gamma_{l+1|l}(\eta_l;x) \mu_l(dx),
\end{align}
where $\Gamma_{l+1|l}(\cdot;x)$ denotes the conditional moment generating function of $X_{l+1},$ given $X_l=x.$ We now fix $\delta>0$ (the exact value to be chosen later) and split the analysis in two cases: 1) $\mathcal A_{l,\delta} = \left\{X_l \in B_{x^\star}(\delta)\right\};$ and 2) $\mathcal A^{\mathrm{c}}_{l,\delta} = \left\{X_l \in B^{\mathrm{c}}_{x^\star}(\delta)\right\}.$

Introduce 
\begin{align}
\Gamma_{l+1| \mathcal A_{l,\delta}}(\eta_l) & : = \mathbb E \left[ 1_{ \|X_l-x^\star\| \leq \delta}\Gamma_{l+1|l}(\eta_l;X_l)\right] \nonumber\\
& = \int_{\|x-x^\star\| \leq \delta} \Gamma_{l+1|l}(\eta_l;x)  \mu_l(dx)\label{eq-Psi-mathcal-A-l}\\
\Gamma_{l+1| \mathcal A^{\mathrm{c}}_{l,\delta}} (\eta_l) & : = \mathbb E \left[ 1_{ \|X_l-x^\star\| > \delta}\Gamma_{l+1|l}(\eta_l;X_l)\right] \nonumber\\
& = \int_{\|x-x^\star\| > \delta} \Gamma_{l+1|l}(\eta_l;x)  \mu_l(dx); \label{eq-Psi-mathcal-A-l-c}
\end{align}
note that 
\begin{equation}
\label{Gamma-two-terms}
\Gamma_{l+1}(\eta_l) = \Gamma_{l+1| A_{l,\delta}}(\eta_l)+ \Gamma_{l+1| A^{\mathrm{c}}_{l,\delta}}(\eta_l).
\end{equation}

\emph{Case 1: $x \in \mathcal A_{l,\delta}$.} Fix $x \in \mathbb R^d$ such that $\|x\|\leq \delta.$ 
We have:
\begin{align}
 \Gamma_{l+1|l}(\eta_l;x)  & = \exp\left(\Lambda (\alpha_l \eta_l; x) + \eta_l^\top (x-\alpha_l g(x)-x^\star)\right) \nonumber\\
& = \exp\left(\Lambda (\alpha_l \eta_l; x) + \eta_l^\top ((I-\alpha_l H^\star)(x-x^\star) -\alpha_l h(x)\right) \nonumber \\
&\leq \exp\left(\Lambda(\alpha_l \eta_l; x^\star ) + L_{\Lambda} \|\eta_l^2\| \|x-x^\star\|  + \alpha_l \|\eta_l\| \|h(x)\|\right) \nonumber \\
\label{eq-mathcal-A-l-1} 
& \,\,\,\,\,\times \exp\left( \eta_l^\top ((I-\alpha_l H^\star)(x-x^\star)\right) \\
\label{eq-mathcal-A-l-2} & \leq \exp\left(\Lambda(\alpha_l \eta_l; x^\star ) + L_{\Lambda} \alpha_l^2 \|\eta_l^2\| \delta^2  + \alpha_l \|\eta_l\| \overline h(\delta) + \eta_{l-1}^\top(x-x^\star)\right),
\end{align}
where in~\eqref{eq-mathcal-A-l-1} we used Lipschitz continuity of $\Lambda$ in $x$, Assumption~\ref{assumption-Lambda-smooth}, and in~\eqref{eq-mathcal-A-l-2} we used the fact that $\|x-x^\star\|\leq \delta$. It follows that
\begin{equation}
\label{eq-recursion-Psi-l-mathcal-A-l}
\Gamma_{l+1| \mathcal A_{\delta}}(\eta_l) \leq \exp\left(\Lambda(\alpha_l \eta_l; x^\star ) + r_0(\lambda, \delta)\right)  \Gamma_l( \eta_{l-1}),
\end{equation}
where $r_0(\lambda, \delta) = L_{\Lambda} a^2 \|\lambda\|^2  \delta^2  + a \|\lambda\| \overline h(\delta).$

\emph{Case 2: $x \in \mathcal A^{\mathrm{c}}_{l,\delta}$}. By strong convexity and Lipschitz smoothness of $f$ in Assumption~\ref{assum-L-mu}, for each $l\geq 1,$ the following holds:
\begin{align}
\label{eq-recursion-xi-alt}
\|X_l - g(X_l) - x^\star\| & \leq \gamma_l \|X_l-x^\star\|, \\
\label{eq-recursion-xi-alt-sup}
& \leq \overline \gamma \|X_l-x^\star\| \end{align}
where $\gamma_l = (1-2\alpha_l \mu + \alpha_l^2 L^2)^{1/2}$, see Appendix~A for the proof, and $\overline \gamma = \sup \{\gamma_l: l=1,2,...\};$ it is easy to verify that $\overline \gamma = \max\{1,  \sqrt{(1-a \mu)^2 + a^2 (L^2 - a^2)}\}.$  

For an arbitrary $x\in \mathbb R^d,$ we have:
\begin{align}
\Gamma_{l+1|l}(\eta_l;x) &  = \exp\left(\Lambda (\alpha_l \eta_l; x) + \eta_l^\top (x-\alpha_l g(x)-x^\star)\right) \nonumber\\
\label{eq-case-2-ineq-1}
& \leq \exp\left(\frac{C_1 \alpha_l^2 \|\eta_l\|^2}{2}\right) \exp\left(\overline \gamma \|\eta_l\| \|x-x^\star\|\right), \\
\label{eq-case-2-ineq-2}
& \leq \exp\left(\frac{C_1 a^2  \|\lambda\|^2}{2}\right) \exp\left(\overline \gamma (l+k_0) \|\lambda\| \|x-x^\star\|\right),
\end{align}
where in~\eqref{eq-case-2-ineq-1} we used the assumption that $Z_k$ is sub-Gaussian, Assumption~\ref{assumption-subGaussian}, for the first term, together with~\eqref{eq-recursion-xi-alt} and Cauchy-Schwartz, for the second term, while in~\eqref{eq-case-2-ineq-2} we exploited~\eqref{eq-eta-l-bound}. Recalling the induced measure $\nu_l,$  we now have
\begin{align}
\label{eq-Gamma-apply-integration-by-parts}
\Gamma_{l+1| \mathcal A^{\mathrm{c}}_{l,\delta}} (\eta_l) & \leq \exp\left(\frac{C_1 a^2  \|\lambda\|^2}{2}\right)
\int_{z \geq  \delta} e^{(l+k_0) \overline \gamma \|\lambda\| z} \nu_l(d z).
\end{align}

The idea of analysing the ``tail'' term $\Gamma_{l+1| \mathcal A^{\mathrm{c}}_{l,\delta}} (\eta_l)$ is the following: by Theorem~\ref{lemma-HPB-Yk}, we know that the probability density $\nu_l$ at a given point $z$ behaves roughly as $e^{-(l+k_0) B z^2}.$ If $\delta$ is sufficiently large, then, for all $z\geq \delta,$ the negative exponential rate of the measure $\nu_l (z)$ is in absolute terms higher than the exponent $(l+k_0) \overline \gamma  \|\lambda\|z.$ Integrating by parts, we obtain that for $\delta = \frac{2 \overline \gamma \|\lambda\|}{B},$ the integral on the right hand-side of~\eqref{eq-Gamma-apply-integration-by-parts} is upper bounded by a constant $K.$  Thus: 
\begin{equation}
\Gamma_{l+1| \mathcal A^{\mathrm{c}}_{l,\delta}} (\eta_l) \leq K \exp\left(\frac{C_1 a^2  \|\lambda\|^2}{2}\right).
\end{equation}
Combining with~\eqref{eq-recursion-Psi-l-mathcal-A-l} and recalling~\eqref{Gamma-two-terms},
\begin{equation}
\label{Gamma-two-terms-1}
\Gamma_{l+1}(\eta_l)  \leq   \exp\left( \Lambda(\alpha_l \eta_l; x^\star )+r(\lambda)\right) \Gamma_l( \eta_{l-1})  \nonumber\\
 +  K \exp\left(\frac{C_1 a^2  \|\lambda\|^2}{2}\right),
\end{equation}
where $r(\lambda) = r_0\left(\lambda,\frac{2 \overline \gamma \|\lambda\|}{B} \right).$ Iterating the preceding recursion, where we exploit the nonnegativity of $\Lambda,$ property 3. from Lemma~\ref{lemma-properties-Lambda}, we obtain:
\begin{align}
\label{eq-unwinded-recursion}
& \Gamma_{k+1}(k\lambda) \leq \exp\left(\sum_{l=1}^k \left(\Lambda(\alpha_l \eta_l; x^\star )+ r(\lambda)\right) \right) \Gamma_1(\alpha_1 \eta_1) \nonumber \\
& + K \exp\left(\frac{C_1 a^2  \|\lambda\|^2}{2}\right)\sum_{l=1}^k    e^{\sum_{j=l}^k \left(\Lambda(\alpha_j \eta_j; x^\star )+ r(\lambda)\right)}\nonumber \\
& \leq (k+1) K \exp\left(\frac{C_1 a^2  \|\lambda\|^2}{2} + a\|\lambda\| \|X_1-x^\star\| \right)  \nonumber \\
& \;\;\;\;\;\;\;\times e^{\sum_{l=1}^k \left(\Lambda(\alpha_l \eta_l; x^\star )+ r(\lambda)\right)}.
\end{align}
Taking the limit, dividing by $k,$ and taking the $\limsup$
\begin{align} 
\label{eq-Lambda-sum-in-RHS}
& \limsup_{k\rightarrow +\infty} \frac{1}{k} \log \Gamma_{k+1}(k\lambda) \leq \nonumber \\
& \;\;\;\;\;\;\;\;\;\;\;\;\;\;\;r(\lambda)+ \limsup_{k\rightarrow+\infty}  \frac{1}{k} \sum_{l=1}^k \Lambda(\alpha_l \eta_l; x^\star).       
\end{align}
Finally, it can be shown that
\begin{equation}
\label{eq-limit-Lambda-sum-equals-integral}
\lim_{k\rightarrow+\infty} \frac{1}{k} \sum_{l=1}^k \Lambda(\alpha_l \eta_l; x^\star ) = \int_{0}^1 \Lambda (a Q D(\theta)Q^\top \lambda; x^\star) d\theta.    
\end{equation}
The proof of~\eqref{eq-limit-Lambda-sum-equals-integral} is provided in Appendix~C. This completes the proof of Lemma~\ref{lemma-LMGF-limsup}.

\section{Conclusions}
We developed large deviations analysis for the stochastic gradient descent (SGD) method, when the objective function is smooth and strongly convex. 
For (strongly convex) quadratic costs, we establish the full large deviations principle. That is, we derive the exact exponential rate of decay of the probability that the iterate sequence generated by SGD stays within an arbitrary set that is away from the problem solution. This is achieved 
for a very general class of gradient noises, that may be iteration-dependent and are required to have a finite log-moment generating function. For generic costs, we derive a tight large deviations upper bound that, up to higher order terms, matches the exact rate derived for the quadratics.



\printbibliography


\newpage

\appendix
\section*{Appendix A.}
\label{app:preliminaries}

\begin{proof} [Proof of recursion~\ref{eq-recursion-xi}]
For any $k\geq 1,$ we have:
\begin{align}
\label{eq-recursion-xi-2}
\|X_{k+1}-x^\star\| & = \|X_{k} - \alpha_k g(X_k) + \alpha_k Z_k - x^\star\|\nonumber\\
& =  \|X_{k} -  x^\star\|^2 -2 \alpha_k (X_k-x^\star)^\top (g(X_k) - Z_k) \nonumber\\
&+ \alpha_k^2 \|g(X_k) - Z_k\|^2 \nonumber\\
& \leq (1-2\alpha_k \mu) \xi_k + 
2 \alpha_k (X_k-x^\star)^\top Z_k \nonumber\\
&+2 \alpha_k^2 \|g(X_k)\|+ 2\alpha_k^2 \|Z_k\|^2\nonumber \\
& \leq \left(1-2\alpha_k \mu + 2 \alpha_k^2 L^2\right) \xi_k + 2 \alpha_k (X_k-x^\star)^\top Z_k \nonumber\\
& + 2 \alpha_k^2 \|Z_k\|^2,
\end{align}
where the first inequality follows from the strong convexity of $f,$ Assumption~\ref{assum-L-mu}, and the fact that, for $a,b\in \mathbb R^d,$ $\|a-b\|^2\leq 2 \|a\|^2+ 2 \|b\|^2,$  and the second inequality follows from the Lipschitz smoothness of $f$, Assumption~\ref{assum-L-mu}. 
\end{proof}

\begin{proof}[Proof of Lemma~\ref{lemma-beta-bounds}]

Fix $l$ and $k$ where $1\leq l \leq k.$ Fix $u,\,v\geq 0.$ From the upper and the lower Darboux sum for the logarithmic function applied to the interval $[l,k]$, we obtain:
\begin{equation}
\label{eq-bounds-on-harmonic}
\log \frac{k+1}{l} \leq \frac{1}{l}+\ldots+\frac{1}{k} \leq \log \frac{k}{l-1}.    
\end{equation}
For the $2$-sum we use the following simple bound $1/l^2 \leq 1/(l(l-1))=1/(l-1)-1/l$ to obtain:
\begin{equation}
\label{eq-bounds-on-2-sum}
\frac{1}{l^2}+\ldots+\frac{1}{k^2} \leq \frac{1}{l-1} - \frac{1}{l}+\ldots + \frac{1}{k-1} - \frac{1}{k} \leq \frac{1}{l-1}.    
\end{equation}

To prove part 1, we use that $1+x \leq e^x$ applied to each of the terms in the product $\beta_{k,l},$ together with the left hand-side inequality of~\eqref{eq-bounds-on-harmonic} and the right hand-side inequality of~\eqref{eq-bounds-on-2-sum}:  
\begin{align}
\beta_{k,l}(u,v) & \leq e^{-a u \sum_{j=l}^k \frac{1}{j+b} + a^2 v \sum_{j=l}^k \frac{1}{(j+b)^2}}\nonumber\\
& \leq e^{-au \log \left(\frac {k+b+1}{l+b}\right)+ \frac{a^2v}{l+b-1}} \nonumber \\
& = \left(\frac{l+b}{k+b+1}\right)^{au} e^{\frac{a^2 v}{l+b-1}}.
\end{align}

To prove part 2, we first note that, since $v\geq 0,$ there holds $\beta_{k,l}(u,v)\geq \beta_{k,l}(u,0),$ i.e., $\beta_{k,l}(u,v)\geq (1-\alpha_k u)\cdots (1-\alpha_l u).$ We now use that, for $x\leq \frac{2}{5},$ $1-x \geq e^{-x-x^2}:$ 
\begin{align}
\beta_{k,l}(u,v) & \geq e^{- a u \sum_{j=l}^k \frac{1}{j+b} - a^2 u^2 \sum_{j=l}^k \frac{1}{(j+b)^2}} \nonumber \\
& \geq \left(\frac{l+b-1}{k+b}\right)^{au} e^{- \frac{a^2 u^2}{l+b-1}}.
\end{align}
This completes the proof of the lemma. 
\end{proof}

\begin{proof} [Proof of~\eqref{eq-recursion-xi-alt}]
Here we prove an alternative recursion on $\|X_k-x^\star\|$, used within the proof of Lemma~\ref{lemma-LMGF-limsup}. Specifically, we show that, for any $k,$
\begin{equation}
\label{eq-recursion-xi-alt-2}
\|X_{k+1}-x^\star\| \leq \gamma_k \|X_k-x^\star\| + \alpha_k \|Z_k\|, 
\end{equation}
where, we recall, $\gamma_k = \left(1-2\alpha_k \mu + \alpha_k^2 L^2\right)^{1/2}.$

From the triangle inequality applied to the Euclidean norm,
\begin{align}
\label{eq-triangle-ineq-xi}
\|X_{k+1}-x^\star\| & = \|X_{k} - \alpha_k g(X_k) + \alpha_k Z_k - x^\star\|\nonumber\\
& \leq  \|X_{k} - \alpha_k g(X_k) - x^\star\| + \alpha_k \|Z_k\|
\end{align}
Exploiting $L$-smoothness and $\mu$- convexity of $f$ for the second term:
\begin{align}
\label{eq-L-mu-exploit}
& \| X_{k} - \alpha_k g(X_k) -  x^\star\|^2 \leq \nonumber\\
& \|X_k - x^\star\|^2 - 2\alpha_k (X_k - x^\star)^\top g(X_k) + \alpha_k^2 \|g(X_k)\| \nonumber\\
& \leq \|X_k - x^\star\|^2 - 2\alpha_k \mu \|X_k - x^\star\|^2 + \alpha_k^2 L^2  \|X_k - x^\star\|^2\nonumber\\
& = \gamma_k^2 \|X_k - x^\star\|^2. 
\end{align}
Taking the square root and replacing in~\eqref{eq-triangle-ineq-xi} yields~\eqref{eq-recursion-xi-alt-2}. 
\end{proof}

\section*{Appendix B.}
\label{app:proof-lemma-HPB}

\begin{proof} [Proof of Lemma~\ref{lemma-HPB-Yk}.]
First, we transform the recursion in~\eqref{eq-recursion-xi} by defining $Y_{k+1} = (k+k_0) \|X_{k+1}-x^\star\|^2,$ to obtain: 
\begin{equation}
\label{eq-recursion-Y}
Y_{k+1} \leq a_k Y_k - b_k \sqrt {k+k_0-1}(X_k-x^\star)^\top Z_k + c_k \|Z_k\|^2,      \end{equation}
where 
\begin{align}
a_k & =\frac{k+k_0}{k+k_0-1} (1-2\alpha_k \mu + 2 \alpha_k^2 L^2)\\
b_k & = \frac{a}{\sqrt{k+k_0-1}} \\
c_k & = \frac{a^2}{k+k_0}.
\end{align}

The key technical result behind Lemma~\ref{lemma-HPB-Yk} is the following upper bound on the tail probability of the $Y_k$ iterates:
\begin{equation}
\label{eq-HPB-Yk-2}
\mathbb P\left( Y_k \geq \epsilon\right) \leq e e^{-B \epsilon}, 
\end{equation}
which holds for each $k\geq 1,$ and $\epsilon \geq 0.$ The result of Lemma~\ref{lemma-HPB-Yk} directly follows from~\eqref{eq-HPB-Yk-2} by taking $\epsilon_k = k \delta^2,$ for each $k.$

Thus, in the remainder of the proof we focus on proving~\eqref{eq-HPB-Yk}. It can be easily verified that, for each $k,$ 
\begin{equation}
a_k = 1- \frac{2a\mu -1}{k+k_0-1} \left(1 - \frac{2 a^2 L^2}{ (2a\mu -1)(k+k_0 -1)}\right).      
\end{equation} 
 Recalling Assumption~\ref{assum-a-mu} and the value of $k_0,$ we see that the above quantity is smaller than $1$ for each $k.$ 

Denote by $\Phi_k$ the moment generating function of $Y_k,$ and by $\Phi_{k+1|k}(\cdot; X_k)$ the moment generating function of $Y_k$ conditioned on $X_k:$ 
\begin{align}
\Phi_k(\nu) & : =\mathbb E \left[\exp(\nu Y_k)\right]\\
\Phi_{k+1|k}(\nu; X_k) & : = \mathbb E \left[\left.\exp(\nu Y_k)\right|X_k\right],
\end{align} 
for $\nu \in \mathbb R;$ note that $\Phi_{k+1}(\nu) = \mathbb E\left[\Phi_{k+1|k}(\nu; X_k)\right],$ for each $\nu \in \mathbb R.$ From the recursion~\eqref{eq-recursion-xi}, we have:
\begin{align}
& \Phi_{k+1|k}(\nu; X_k)  = \nonumber \\
& \exp( a_k \nu Y_k) \mathbb E\left[  \left.\exp( - b_k \sqrt {k+k_0-1}(X_k-x^\star)^\top Z_k + c_k \|Z_k\|^2) \right|X_k\right] \nonumber \\
& \leq  \exp( a_k \nu Y_k)\, \left( \mathbb E\left[  \left.\exp( - 2 b_k \nu \sqrt {k+k_0-1}(X_k-x^\star)^\top Z_k)  \right| X_k\right] \right)^{1/2}  \times \nonumber \\
& \,\,\,\,\,\,\,\left( \mathbb E\left[  \left.\exp(2 c_k \nu \|Z_k\|^2) \right|X_k\right] \right)^{1/2}\nonumber\\
& \leq \exp( a_k \nu Y_k)\, \exp(2 b_k^2 \nu^2 Y_k)\, \left( \mathbb E\left[ \left. \exp(2 c_k \nu \|Z_k\|^2) \right|X_k\right] \right)^{1/2}
\end{align}
Recalling~\eqref{part-LMGF-Z_k-squared}, the last term is finite for $\nu \leq 1/(2 a^2 C_2) =: B_0,$ and for such $\nu$, the corresponding value is equal to $\exp(C_2 c_k \nu )$. Thus, for each $\nu \leq B_0,$
\begin{equation}
\label{eq-Phi-key-inequality}
\Phi_{k+1|k}(\nu; X_k) \leq  \exp(\nu (a_k + 2 b_k^2 \nu) Y_k + C_2 c_k \nu). 
\end{equation}

It is easy to see that  $B\leq B_0.$ Consider $\nu \leq B$. Taking the expectation on both sides of~\eqref{eq-Phi-key-inequality}, the following recursive inequality on $\Phi_k$ is obtained for any $\nu \leq B$ and any $k\geq 1:$
\begin{equation}
\label{eq-Phi-recursion}
\Phi_{k+1}(\nu) \leq  \Phi ((a_k + 2 b_k^2 B) \nu) \exp(C_2 c_k \nu). 
\end{equation}

From this point, the proof proceeds similarly as in~\cite{Harvey2019}, i.e., by induction, and using $k=1$ as the base, it can be shown that, for each $\nu \leq B,$
\begin{equation}
\label{eq-MGF-bound-Y-k}
\Phi_{k}(\nu) \leq e^{\frac{\nu}{B}}.    
\end{equation}
By exponential Markov, from~\eqref{eq-MGF-bound-Y-k}, for each $\nu \leq B,$
\begin{equation}
\mathbb P\left(Y_k\geq \epsilon\right)   \leq \mathbb E\left[\exp{ \nu Y_k} e^{-\nu \epsilon} \right]. 
\end{equation}
Taking $\nu = B$ yields the desired result.
\end{proof}

\section*{Appendix C.}
\label{app:proof-Lambda-sum-equals-Lambda-integral}

\begin{proof} [Proof of~\eqref{eq-limit-Lambda-sum-equals-integral} ]

Introduce step-wise constant function $s_k: [0,1]\mapsto \mathbb R$, defined by
\begin{equation}
s_k(\theta)=\left\{
\begin{array}{cc}
\Lambda (\alpha_l \eta_l; x^\star), & \mathrm{for\,} \frac{l-1}{k} < \theta \leq \frac{l}{k} \\
0   , & \mathrm{for\,}\theta=0 
\end{array}
\right..    
\end{equation}

It is easy to verify that the integral of $s_k$ over $[0,1]$ equals the desired sum in the right hand-side of~\eqref{eq-Lambda-sum-in-RHS}, i.e.,
\begin{equation}
\int_0^1 s_k(\theta) = \frac{1}{k}\sum_{l=1}^k \Lambda (\alpha_l \eta_l; x^\star).    
\end{equation}
We next show that
\begin{equation}
\label{eq-limit-s-k}
\lim_{k\rightarrow +\infty} s_k(\theta) = \Lambda (a Q D(\theta)Q^\top \lambda; x^\star), 
\end{equation}
where $D(\theta)$ is as defined in the claim of the theorem. To show the preceding limit, note that, for each $\theta\in (0,1]$, 
\begin{equation}
s_k(\theta) = \Lambda ( k \alpha_{l_k} Q D_{k,l_k} Q^\top \lambda; x^\star)
\end{equation}
where $[D_{k,l_k}]_{ii} = \beta_{k,l_k}(\rho_i,0),$ $l_k$ is the index of the interval in the definition of $s_k$ to which $\theta$ belongs, $l_k=\min\{l=1,..,k: \theta \leq \frac{l}{k}\},$ and $\rho_i$ is, we recall, the $i$-th eigenvalue of $H^\star.$

Using the bounds from Lemma~\ref{lemma-beta-bounds}, it is easy to establish the by sandwiching argument that 
\begin{equation}
\lim_{k\rightarrow \infty}  k\alpha_{l_k} \beta_{k,l_k} (\rho_i,0)= a \theta^{a\rho_i-1}.
\end{equation}
The limit in~\eqref{eq-limit-s-k} now follows by the continuity of $\Lambda(\cdot;x^\star),$ which follows by convexity of $\Lambda(\cdot; x^\star),$ Lemma~\ref{lemma-properties-Lambda}.

Using the fact that $s_k$ can be uniformly bounded for all $k$ and $\theta \in [0,1]$, we can exchange the order of the limit and the integral, to obtain: 
\begin{equation}
\lim_{k\rightarrow +\infty} \int_0^1 s_k(\theta) d \theta =  \int_0^1 \lim_{k\rightarrow +\infty} s_k(\theta) d \theta =  \int_0^1 \Lambda (a Q D(\theta)Q^\top \lambda; x^\star),    
\end{equation}
establishing the claim of the lemma.
\end{proof}

\section*{Appendix D.}
\label{app:derivations-from-remarks}

\begin{proof} [Derivations with Remark~6]
Consider function $\overline{\Psi}(\lambda) = \Psi^\star(\lambda)+r(\lambda)$ in Lemma~\ref{lemma-LMGF-limsup}.  
 We derive here  a lower bound on 
 rate function $\overline{I}$ in Theorem~~\ref{theorem-LDP-general} that does not explicitly depend on $H(x^\star)$. 
 In view of the fact that $\overline{I}$ is the Fenchel-Legendre transform of $\overline{\Psi}$, a lower bound on $\overline{I}$ is readily obtained by deriving 
 an upper bound on $\overline{\Psi}(\lambda).$ Note that $r(\lambda)$ does not explicitly depend on $H(x^\star)$, hence we only need to 
 derive an upper bound on $\Psi^\star$. 
 By Assumption~\ref{assumption-subGaussian}, we have, for any $\theta \in [0,1]$, that 
 $\Lambda(a Q D(\theta) Q^\top \lambda; x^\star)
  \leq \frac{C_1\,a^2}{2} \lambda^\top (Q D Q^\top)^2 \lambda $
  $\leq \frac{C_1\,a^2}{2}\|\lambda\|^2\,\|D(\theta)\|^2$, 
  where we recall that $\|\cdot\|$ denotes the 2-norm
   of its vector or matrix argument. 
  Next, note that 
  $\|D(\theta)\| \leq \theta^{a\,\mu-1}$, for all $\theta \in [0,1]$, because 
  all eigenvalues $\rho_i$'s of $H(x^\star)$ belong 
  to the interval $[\mu,L]$. 
  Therefore, we obtain:
  \begin{eqnarray*}
  &\,& \Psi^\star(\lambda) \leq \frac{C_1\,a^2}{2}\|\lambda\|^2 
  \int_0^1 \theta^{2 a\,\mu-2}d \theta 
  = \frac{C_1 \,a^2 }{2(2\,a\mu-1)}\|\lambda\|^2.
  \end{eqnarray*}
  \end{proof}
  
\begin{proof} [The case of random initial iterate $X_1$]
Recall that, by definition, $\Gamma_1(\lambda) = \mathbb E\left[e^{\lambda^\top(X_1-x^\star)}\right].$ When $X_1$ is random, $\Gamma_1$, as a function of $\lambda,$ is therefore the log-moment generating function of $X_1-x^\star.$ Provided its domain is $\mathbb R^d,$ all arguments in the proof of Theorem~\ref{theorem-LDP-general} remain the same. In particular, in eq.~\eqref{eq-unwinded-recursion}, the factor $e^{\|\lambda\|\|X_1-x^\star\|}$ would be replaced by a (finite-valued) function (of $\lambda$), and the subsequent results would be unaltered; a similar comment applies for the statement and the proof of Lemma~\ref{lemma-HPB-Yk}.  
\end{proof}

\section*{Appendix E.}
\label{app:proof-for-quadratic}

\begin{proof}[Proof of Theorem~\ref{theorem-quadratic-full-LDP}]
It is easy to show that for the assumed quadratic form, the iterates $X_k$ have the following representation:
\begin{equation}
X_{k+1} =    A_{k0}X_1  + \sum_{l=1}^k \alpha_l A_{k,l+1}Z_l,
\end{equation}
where $A_{k,l} = \prod_{j=l}^k (I-\alpha_j H)$. By the assumption that the noise realizations at different times are independent and with a constant distribution, we obtain:
\begin{equation}
\label{eq-MGF-X-k-quadratic}
\Gamma_{k+1} (\lambda)= e^{\lambda^\top X_1 }  e^{\sum_{l=1}^k \Lambda( \alpha_l A_{k,l+1} \lambda)}.     
\end{equation}
The proof now follows from Lemma~\ref{lemma-GE-auxilliary} and the limit established in~\ref{eq-limit-Lambda-sum-equals-integral}.
\end{proof}

\section*{Appendix F. Numerical results}
\label{app:numerical-results}

We now illustrate the achieved results through a numerical simulation. 
We consider a strongly convex quadratic cost function $f: {\mathbb R}^d \rightarrow \mathbb R$, defined by $f(x) =\frac{1}{2}x^\top A x + b x$, $d=10$, where the symmetric $d \times d$ matrix $A$ and the $d \times 1$ vector $b$ are generated randomly. Specifically, we generate the entries of $b$ mutually independently, according to the standard normal distribution. The matrix $A$ is generated as follows. We let $A = Q \Lambda Q^\top$, where $Q$ is the matrix whose columns are the orthonormal eigenvectors of matrix $(B+B^\top)/2$, and the entries of $B$ are drawn mutually independently from the standard normal distribution; the matrix $\Lambda$ is the diagonal matrix whose diagonal entries are drawn from the uniform distribution on the interval $[1,2]$. Clearly, the optimal solution for the problem equals $x^\star = A^{-1}b$.

We consider the gradient noise that is generated in an i.i.d. manner over iterations and over the gradient noise vector elements, independently from the solution iterate sequence. Two different noise distributions per gradient noise entry are considered, such that the per-entry noise variance is kept equal for the two distributions, equal to $\sigma^2$. In this way, we evaluate the effects of higher order moments on the performance of SGD. The first distribution is zero-mean Gaussian with variance $\sigma^2$. The second distribution is the zero-mean Laplacian with the same variance. We set $\sigma^2=0.04.$

We numerically estimate, via Monte Carlo simulations, the probability $P \left( \|X_k-x^\star\| > \delta\right)$ along iterations~$k=1,2,...$ We denote the corresponding numerical estimate by~${p}_k$. 
 Two different values of $\delta$ are considered, $\delta = 0.3$,  and $\delta = 0.03$. For each Monte Carlo run, $X_1$ is set to the zero vector. For the numerical example here, 
 $\|x^\star\|=2.342$, and hence $\delta = 0.3$ corresponds to the relative 
 error level  
 $\delta/\|x^\star\| \approx 0.13$, while $\delta = 0.03$  
 corresponds to $\delta/\|x^\star\| \approx 0.013$. 
  Figure~1 plots ${p}_k$ versus iteration counter $k$ 
 (in linear scale for the horizontal axis, and $\mathrm{log}_{10}$-scale for the vertical axis) for the Gaussian noise case (blue line) and 
 the Laplacian noise case (red line). 
 The top Figure is for $\delta = 0.3$, and 
 the bottom Figure is for $\delta=0.03$. We can see that, 
 for a large value of $\delta$, the two curves 
 are very different: the Laplacian gradient noise case 
 leads to a worse performance. 
  This is because, for large $\delta$, 
  the argument $\lambda$ of the LMGF $\Lambda$  
  that corresponds to the minimizer in 
  the rate function value $I^\star$ is large (see Theorem~4), and hence 
  higher order polynomial coefficients ($\sim \lambda^3$ and higher)
   play a significant role. As the higher order moments of the Gaussian and Laplace distributions are very different (equal to zero for the Gaussian and strictly positive for the Laplacian), the result is 
   the different large deviations performance (worse for the Laplacian case) as seen in Figure~1, top. 
    On the other hand, for a small value of $\delta$ (bottom Figure), 
    the argument  $\lambda$ of the LMGF that corresponds to 
    the minimizer in the rate function expression $I^\star$ is small, 
    and hence only the first two order polynomial coefficients 
    of $\Lambda$ play a significant role. As the two distributions here are both zero mean and have equal variance (hence having equal first and second order moments), the large deviation performance for the two noises 
    matches, as seen in Figure~1, bottom. This behavior 
    is in accordance with the theory derived.

 \begin{figure}[thpb]
      \centering
      \includegraphics[width=9cm, angle=-90]{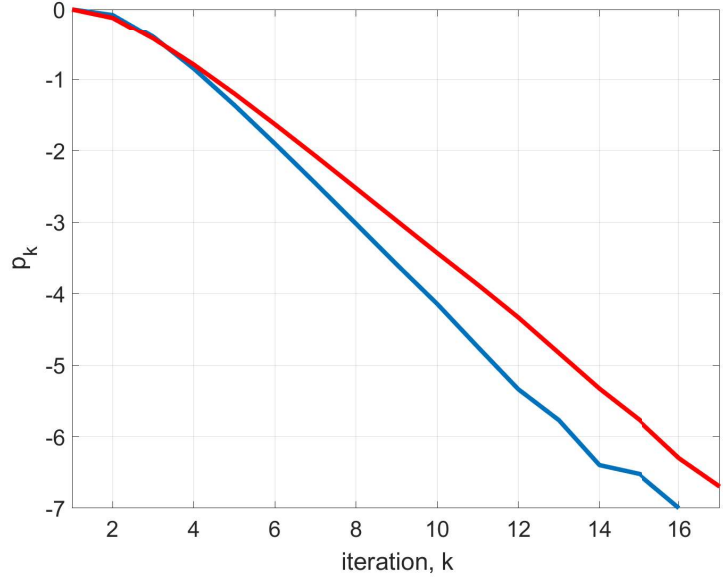}
     \includegraphics[width=9cm, angle=-90]{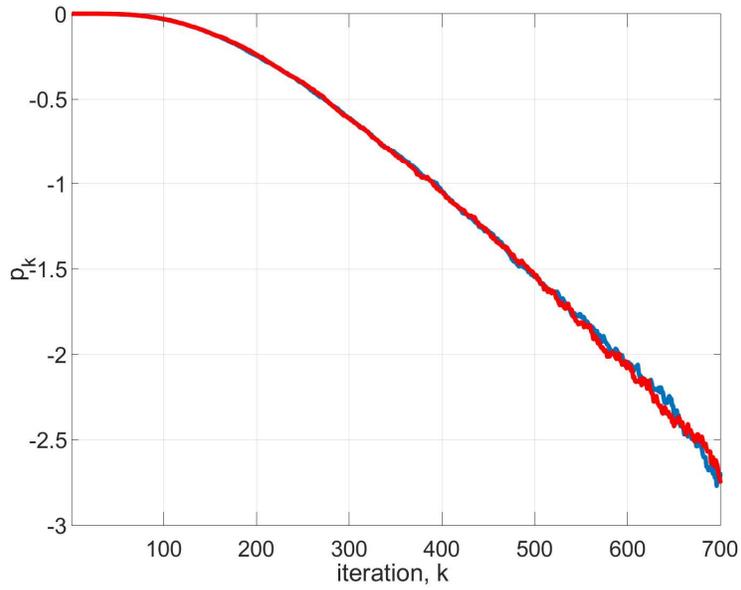}
    \caption{Monte Carlo estimate of $P \left( \|X_k-x^\star\| > \delta\right)$ along iterations~$k=1,2,...$ for 
   SGD with Gaussian (blue line) and Laplacian (red line) gradient noise with equal per-entry variance $\sigma^2=0.04$. 
  Top Figure: $\delta=0.3$; Bottom Figure: $\delta=0.03$.}
\end{figure}

\vfill

\end{document}